\documentclass[conference]{IEEEtran}
\IEEEoverridecommandlockouts

\usepackage{amsmath,amssymb,amsfonts,amsthm}
\usepackage{algorithm}
\usepackage{soul}
\usepackage[noend]{algpseudocode}
\usepackage{graphicx}
\usepackage{textcomp}
\usepackage{xcolor}
\usepackage{todonotes}
\usepackage[caption=false, font=footnotesize]{subfig}

\newtheorem{theorem}{Theorem}

\newtheorem{lemma}{Lemma}

 \usepackage{tikz}
 \tikzset{>=latex} 
 \usepackage{pgfplots} 
 \usepackage{xcolor}
 \usepackage[outline]{contour} 
 \contourlength{1.2pt}
 \usetikzlibrary{positioning,calc}
 \usetikzlibrary{backgrounds}
\usepackage{adjustbox} 

 \pgfmathdeclarefunction{gauss}{3}{
   \pgfmathparse{0.2/(#3*sqrt(2*pi))*exp(-((#1-#2)^2)/(2*#3^2))}%
 }
 \pgfmathdeclarefunction{attacked}{4}{
   \pgfmathparse{(#1<#2)*gauss(#1,#2,#3) + and(#1>=#2,#1<#2+#4)*gauss(#2,#2,#3) + (#1>=#2+#4)*gauss(#1,#2+#4,#3)}
 }

 \colorlet{mydarkblue}{blue!30!black}
 \colorlet{mydarkred}{red!30!black}

 \usepgfplotslibrary{fillbetween}
 \usetikzlibrary{patterns}
 \pgfplotsset{compat=1.12} 

\usepackage[style=ieee]{biblatex} 
\begin{document}



\title{Robust Federated Personalised Mean Estimation for the Gaussian Mixture Model
\thanks{MM and VP acknowledge support of the Department of Atomic Energy, Government of India, under project no. RTI4001.  The work of SD was supported in part by NSF grants 2139304, 2146838}
}

\author{\IEEEauthorblockN{Malhar A. Managoli}
\IEEEauthorblockA{\textit{School of Technology \& Comp. Sc.} \\
\textit{Tata Institute of Fundamental Research}\\
Mumbai, India \\
malhar.managoli@tifr.res.in}
\and

\IEEEauthorblockN{Vinod M. Prabhakaran}
\IEEEauthorblockA{\textit{School of Technology \& Comp. Sc.} \\
\textit{Tata Institute of Fundamental Research}\\
Mumbai, India \\
vinodmp@tifr.res.in}
\and
\IEEEauthorblockN{Suhas Diggavi}
\IEEEauthorblockA{\textit{Department of Electrical Engineering} \\
\textit{University of California}\\
Los Angeles, USA \\
suhas@ee.ucla.edu}
}

\maketitle

\begin{abstract}
     Federated learning with heterogeneous data and personalization has received significant recent attention. Separately, robustness to corrupted data in the context of federated learning has also been studied. In this paper we explore  combining personalization for heterogeneous data with robustness, where a constant fraction of the clients are corrupted. Motivated by this broad problem, we formulate a simple instantiation which captures some of its difficulty. We focus on the specific problem of personalized mean estimation where the data is drawn from a Gaussian mixture model. We give an algorithm whose error depends almost linearly on the ratio of corrupted to uncorrupted samples, and show a lower bound with the same behavior, albeit with a gap of a constant factor.
\end{abstract}


\section{Introduction}

 Federated learning (FL) is a distributed system approach to collaboratively build machine learning models from multiple clients,   without directly sharing limited local data \cite{mcmahan2017communication,kairouz2019advances}. There are two challenges. One is that the local data could be statistically heterogeneous, which means that we need to build individual (personalized) models suitable for the local statistics, and therefore the challenge is to collaborate across clients with such heterogeneous data. The second is that, since the clients are distributed, some of them could be compromised or be malicious and therefore can disrupt the collaboration. Therefore, the fundamental questions is how to build robust learning methods despite heterogeneity.

In this broad problem, we focus on robust distributed personalized mean estimation for heterogeneous data where the client wants to estimate the mean of its underlying distribution. The general problem is challenging, and so we formulate and study a specific case that captures several of the important and challenging aspects. For concreteness we {\sf (i)} study a specific heterogeneity model, based on Gaussian mixture model, {\sf (ii)} use the limited local data as ``verified'' samples for the clients to combine them with collaboration adaptively, {\sf (iii)} devise robust filtering algorithms cognizant of the statistical model albeit without knowing its parameters {\sf (iv)} formulating a performance criterion for which we give analytical guarantees of performance for a chosen Byzantine model.

 Specifically in our robust personalized mean estimation, the data is drawn from a Gaussian mixture model and each client is interested in estimating the mean of the mixture component their data belongs to. We consider an adversary who can overwrite the data of a fraction of the clients. A server collates  data from all the clients to provide information which may assist the clients in their task. Each (uncorrupted) client combines this information with their own limited, but ``verified'' local data to produce their estimates. Our goal is to understand the value of using the data of other clients which, while heterogeneous and partially corrupted, may be plentiful; we focus specifically on the case of large number of clients and put no limits on the client-server communication. We formulate a min-max problem in which the objective is to minimize the mean squared error of the average uncorrupted client's personalized estimate maximized over the model and the actions of the adversary.

One difficulty of being robust in the heterogeneous setting is that even when the corrupted clients constitute a small fraction of the total number of clients, they can be more numerous than one of the components of the mixture.
Any one client (or even an entire component) can be made to achieve no benefit from the collaboration. However, in this case most clients get significant benefits by collaborating.
Therefore, we try to minimise the error incurred by the average client. And we show that with respect to this error, collaboration is almost always beneficial. Another challenge is handling variations induced by the adversary which maybe indistinguishable from natural variations (due to heterogeneity). Furthermore, in forming the personalized estimate, the local samples (which are known to be uncorrupted) need to be judiciously combined with information obtained from samples of other users (which are partially corrupted).

\noindent\textbf{Contributions:} (see Sections \ref{sec:MainRes}, \ref{sec:LowerBnd} and \ref{sec:AlgoAnalysis}). 
\begin{itemize}
\item We obtain a lower bound on the error that any algorithm must incur (see Theorem \ref{thm:lowerbound}). We do this by considering a relaxation where the algorithm also has access to component labels for the samples (including those of the adversary). This effectively reduces the problem to several single component problems for which we design and analyse an appropriately chosen attack (see Section \ref{sec:LowerBnd}). 
\item We propose a filter-then-cluster algorithm for robust personalized mean estimation.
The algorithm has several broad components. First, the filtration and clustering steps are designed so that each corrupted data point can only affect one component (under a minimum separation guarantee). 
Second, the algorithm gives estimates for the mean and the quality of the mean estimate, for each component. This allows the clients to know how much to rely on the server's estimate and judiciously combine it with its local uncorrupted samples to form its estimate. 
\item We prove an asymptotic upper bound on the error incurred by the algorithm, matching the lower bound to within a constant factor 
We do this by first showing that each corrupted data point can affect only one component, and then by analysing the single component version of the problem (this was inspired by the lower bound).
For analysing the single component problem, we give bounds on the errors in server's mean estimate and quality estimate. Based on these we bound the error in the client's final estimate.
\end{itemize}

\noindent\emph{Related work:} Classical works on robust statistics~\cite{tukey1960survey,huber1964robust} initiated the study of designing estimators which are robust to some of the data being corrupted. Recent works have proposed efficient algorithms for high-dimensional robust estimation~\cite{diakonikolas2019robust,lai2016agnostic,charikar2017learning,}. A related line of investigation considers learning from batches of data where a fraction of batches are corrupted by an adversary~\cite{qiao2018learning,jain2020optimal,chen2020learning,jain2021robust}, however this is mainly for the homogeneous case where the samples are drawn from the same underlying distribution. Ghosh et al.~\cite{ghosh2019robust} considered robust federated learning in the heterogeneous setting. However, their objective is different from ours; their personalisation is on the component level, not at the client level. Also, they try to minimise the error of the worst component. In another line of work (see \cite{DD-ICML-2021,data2023byzantine} and references therein), robust mean estimation with heterogeneous Byzantine clients is studied. The differences are that they considered overall mean estimation rather than personalized for each client, and the heterogeneity model was based on gradient function properties. There is also extensive  work on personalized federated learning (see \cite{OGDD-ICLR-2023} and references therein), where the problem is to build individual learning models, through collaboration, but most of these works are without Byzantine adversaries. 

\section{Problem Statement}

Consider the following $k$-component univariate Gaussian mixture distribution $\mathcal{G}=\frac{1}{k} \sum_{i\in[k]} \mathcal{N}(\mu_i,1)$ with equal mixture weights, where $\mu_i\in{\mathbb R}, i\in[k]:=\{1,\ldots,k\}$ are the unknown means of the components. The variances are identical and assumed to be known (these are set to unity without loss of generality). Clients independently observe a single realization each\footnote{Or equivalently (up to a scaling), each client makes the same number of independent observations whose sample average may serve as a sufficient statistic.} according to $\mathcal{G}$ and are interested in estimating the mean of the component they belong to. In order to do this, they may take advantage of the observations of the other clients through a server to whom all clients report their observations\footnote{In this short paper, we do not impose any constraints on communication between the clients and the server.}. However, an adversary corrupts a subset of the clients whose identity is unknown to the uncorrupted clients and the server, which is assumed to be uncorrupted. We consider an adversary who is not privy to the observations of the uncorrupted clients, but knows the component means exactly (note that these means may at least be estimated from the observations of the corrupt clients; we assume here that the adversary knows them exactly). The adversary may replace the observations of the corrupted clients with any value of its choosing. This is essentially the non-adaptive, additive contamination model of~\cite{huber1964robust}. We study a min-max formulation where the goal of the server and (uncorrupted) clients is to minimize the mean squared error (MSE) incurred by the average uncorrupted client maximized over the choice of the component means (subject to a minimum separation assumption) and the actions of the adversary.

We assume that there are $(k+c)m$ clients. Each client is independently corrupted with probability $c/(k+c)$; let $U\subseteq[(k+c)m]$ denote the set of uncorrupted clients. The samples of uncorrupted clients are drawn independently according to $\mathcal{G}$; the samples of the corrupted clients are chosen by the adversarial algorithm, denoted by $\mathcal{A}$, which takes as input the component means $\{\mu_i\}_{i\in[k]}$. Let $X_i$ denote the sample of the $i$-th client. The server's algorithm $\mathcal{S}$ takes the set of all samples and produces an output which is broadcast to the clients. Each uncorrupted client runs an algorithm $\mathcal{C}$ whose input is the client's own local sample and the output of $\mathcal{S}$. Consider the following min-max problem, where the objective is the average MSE of the uncorrupted clients' estimates of their component means:
\begin{align*}
	K(m):=
 \min_\mathcal{S,C}&\max_{\substack{\mathcal{A},\;\{\mu_i\}_{i\in[k]}:\\ |\mu_i-\mu_j|\geq D, i\neq j}}
  \mathbb{E}\Bigg[\frac{1}{|U|}\sum_{i\in U}\Big(\mu_i-\\
  &\qquad\mathcal{C}\left(X_i, \mathcal{S}\left(\left\{X_i\right\}_{i\in[(k+c)m]} \right)\right)\Big)^2\Bigg].
\end{align*}
Notice that in the maximization, the components means are constrained to have a separation of at least $D$.

\section{Results}
\label{sec:MainRes}

Before stating the error bound achieved by our algorithm, we state a lower bound to serve as a benchmark.
\begin{theorem}[Lower Bound]
	For the above mean estimation problem, $K(m)\ge\frac{1}{\sqrt{8\pi}}\frac{c}{k}$ for all $m$, for $c/k<\sqrt{2\pi}$.\label{thm:lowerbound}
\end{theorem}
Note that, without collaboration, each client incurs an \textcolor{black}{expected} MSE of 1. 
The above lower bound implies that collaboration may help as long as not too many clients are corrupt, and the MSE must worsen at least linearly when the ratio of corrupt clients to genuine clients increases. The proof outline of Theorem \ref{thm:lowerbound} is given in Section \ref{sec:LowerBnd}.
We now state the error bound for our algorithm.
\begin{theorem}[Asymptotic Upper Bound]
\label{thm:upperbound}
    Given an instance of the above mean estimation problem such that the component means have minimum separation $D\ge9\Delta$ and $\Delta>1.5$,
	\begin{align*}
		K_\infty\le\begin{cases}
		    C_0(\Delta)\frac{c}{k}+C_1(\Delta) & \frac{c}{k}>\varepsilon_{cr}\\
            C_2(\Delta)\left(\frac{c}{k}\right)^{2/9} + 2\Phi\left(-\frac{3\Delta}{2}\right)C_3\Delta^2 & \frac{c}{k}\le\varepsilon_{cr}
		\end{cases}
	\end{align*}
	Here, $K_\infty=\lim_{m\rightarrow\infty}K(m)$. $C_0$ is a function of $\Delta$ which is close to 1 and tends to 1 as $\Delta\rightarrow\infty$. $\varepsilon_{cr}$, $C_1,C_2$ are functions of $\Delta$ which are $\ll 1$ and go to 0 as $\Delta\rightarrow\infty$. 
\end{theorem}
\begin{figure} 
    \centering
    \subfloat[\label{fig:err_vs_c/k}]{
       \includegraphics[width=0.7\linewidth]{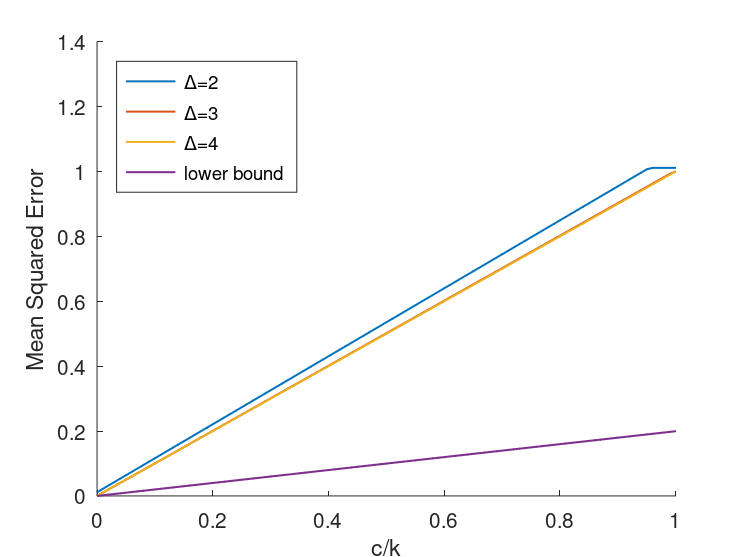}}
    \hfill
    \subfloat[\label{fig:err_vs_delta}]{%
        \includegraphics[width=0.7\linewidth]{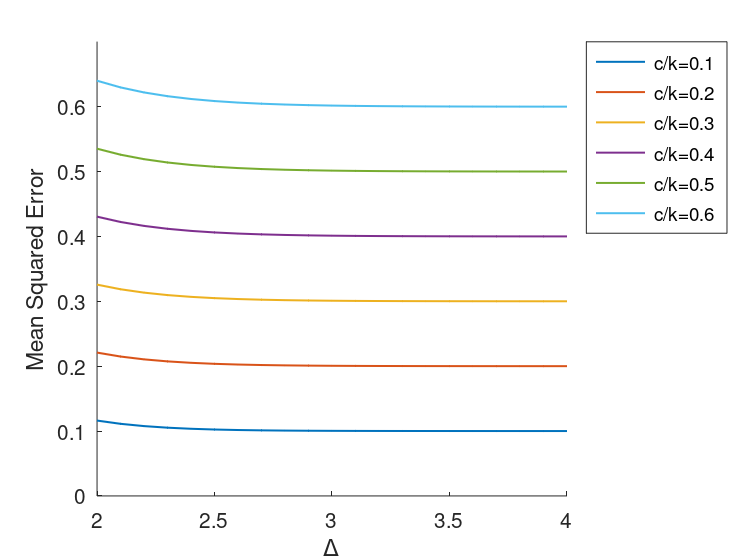}}
    \caption{(a) A plot of our error bound against $c/k$ for various values of $\Delta$. The lower bound is shown in purple. The $\Delta=3$ and $\Delta=4$ curves are very similar and hard to distinguish by eye. 
    (b) A plot of our error bound against $\Delta$ for various values of $c/k$. Note that all the curves are flat for $\Delta>2.5$. Thus, larger $\Delta$ are unnecessary to achieve our upper bound.}
    \label{fig:plots}
\end{figure}
Note that $\varepsilon_{cr}$ is very small and decreases quickly to 0 as $\Delta\rightarrow\infty$ (for $\Delta= 2,3,4$, $\varepsilon_{cr} \approx 8.1\cdot10^{-6}, 1.2\cdot10^{-9},1.4\cdot10^{-14}$ respectively.) 
Thus, for all practical cases of interest, the bound is $C_0(\Delta)\frac{c}{k}+C_1(\Delta)$.
It turns out that $C_1(\Delta)\ll 1$ and $C_0(\Delta)\approx 1$, so our upper bound is approximately $c/k$ which is a factor of $\approx\sqrt{8\pi}$ off from the lower bound, see Figure \ref{fig:err_vs_c/k}.
In Figure \ref{fig:err_vs_delta}, it can be seen that the upper bound's dependence on $\Delta$ quickly approaches its limiting behaviour, there being little difference in the bound for $\Delta\geq 3$.


These results show when collaboration is helpful in the presence of corruptions. Without collaboration, the average client's mean squared error would be 1. With collaboration however, they can achieve an asymptotically smaller error as long as the ratio of corrupted to uncorrupted clients ($c/k$) is less than $1/C_0(\Delta)\approx1$. The lower bound shows that the linear dependence of the MSE with collaboration on $c/k$ is optimal. 

There is still a factor of $\approx\sqrt{8\pi}$ between the upper and lower bounds, which we believe can be improved via better algorithms; the challenge is for it to be tractably analyzable, as even our simpler algorithm is challenging to analyze. 

\section{Lower Bound}
\label{sec:LowerBnd}

To derive a lower bound on $K(m)$, we consider a relaxation of the problem in which the algorithm has additionally access to the component labels of the observations. The labels of the uncorrupted observations are correct, while those of corrupted one are chosen by the adversary. Since the algorithm may ignore the component labels, a lower bound for this relaxation is clearly also a lower bound on $K(m)$. 

This relaxation allows us to consider the components separately and our bound will hold for an arbitrary choice of $\{\mu_i\}_{i\in[k]}$ (specifically, one which satisfies the minimum separation requirement). 
\textcolor{black}{We lower bound $K(m)$ by focusing on a specific adversary (instead of maximising over all adversaries):} for non-negative reals $\varepsilon_i$, $i\in[k]$ which add up to $c$, consider the adversary who selects $\mu_i'\in[\mu_i-\sqrt{2\pi}\varepsilon_i,\mu_i], i\in[k]$ arbitrarily and produces the purported observations of the corrupted clients independently using the following sampling strategy\footnote{Note that this adversary is randomized and the resulting MSE is a lower bound on that of the worst case (maximizing) adversary.}: The label $i\in[k]$ is sampled with probability $\varepsilon_i/c$, and conditioned on label $i$ being sampled, the observation is sampled using the probability density function (p.d.f) $\psi(x)=(f_{\varepsilon_i}(x-\mu'_i)-\varphi(x-\mu_i))/\varepsilon_i$, where $\varphi$ is the standard normal p.d.f. and 
\begin{align*}
    f_{\varepsilon_i}(x)
    =\begin{cases}
		\varphi(x)& x\le 0\\
		\frac{1}{\sqrt{2\pi}}& 0<x\le \sqrt{2\pi}\varepsilon_i\\
		\varphi(x-\sqrt{2\pi}\varepsilon_i)& x> \sqrt{2\pi}\varepsilon_i\\
	\end{cases}
\end{align*}
It is easy to verify that $\psi$ is a p.d.f. if $\mu_i'\in[\mu_i-\sqrt{2\pi}\varepsilon_i,\mu_i]$.
With this, the collection of $(k+c)m$ samples reported to the server can be thought of as drawn from 
the distribution $\frac{1}{k+c} \sum_{i\in[k]} f_{\varepsilon_i}(x-\mu'_i)$.
Hence, the best that the server can hope to learn about component-$i$ is the pair of parameters $(\varepsilon_i,\mu'_i)$. Since $\mu_i'$ was chosen arbitrarily in $[\mu_i-\sqrt{2\pi}\varepsilon_i,\mu_i]$, this amounts to learning that $\mu_i$ lies in $[\mu'_i,\mu'_i+\sqrt{2\pi}\varepsilon_i]$. To obtain a lower bound, assume that the server learns this information. Then, the uncorrupted clients in the $i$-th component (who have their own observation) are solving the following min-max parameter estimation problem (where we assume $\mu'_i=0$ without loss of generality):

Given a single sample $X$ from $\mathcal{N}(\theta,1), \theta\in[0,\sqrt{2\pi}\varepsilon_i]$, minimise the mean squared error
\begin{align*}
	R(\varepsilon_i)=\min_{\hat{\theta}(.)}\max_{\theta\in[0,\sqrt{2\pi}\varepsilon_i]}\mathbb{E}_{X\sim\mathcal{N}(\theta,1)}[(\hat{\theta}(X)-\theta)^2].
\end{align*}
Lower bounding this minimax MSE using the Bayesian MSE for which the worst case prior (i.e., the one which minimises the Fischer information) is known to be the squared cosine distribution \cite{bercher2009minimum}, we obtain
\[ R(\varepsilon_i) \ge \frac{\varepsilon_i^2}{\varepsilon_i^2+2\pi}=:h(\varepsilon_i).\]
Therefore, $K(m)\ge\max_{\{\varepsilon_i\geq 0\}_{i\in[k]}:\sum\varepsilon_i\leq c} \quad\frac{1}{k}\sum_{i\in[k]}h(\varepsilon_i).$
The function $h(\varepsilon)$ is convex for $\varepsilon\leq\sqrt{2\pi/3}$ and concave thereafter. A detailed analysis 
via the Lagrange multiplier method shows that the above expression is maximised when the adversary splits its budget evenly across $c/\sqrt{2\pi}$ components leaving the rest alone,
yielding a lower bound of $\frac{1}{\sqrt{8\pi}}\frac{c}{k}$.

\begin{figure}
    \centering
  \includegraphics[clip=true, trim={130 340 100 330}, width=10cm]{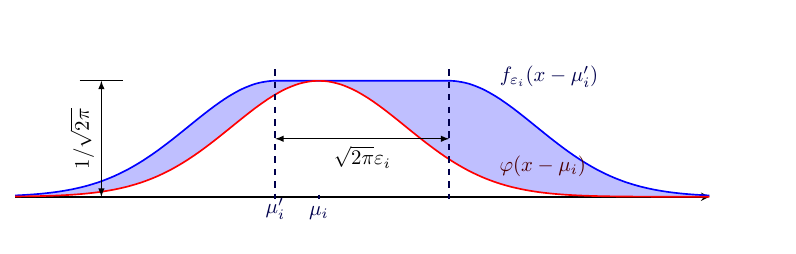}
   \caption{Attack on the $i$-th component. The shaded area is $\varepsilon_i$ and represents the additional mass introduced to this component by the adversary.}
    \label{fig:lowerbound}
\end{figure}

\section{Algorithm and Asymptotic Analysis}
\label{sec:AlgoAnalysis}

The algorithm consists of two subroutines. The first is {\sc Robust Clustering}, which the server uses to compute component mean estimates based on clients' data. The second is {\sc Combine Estimates}, which uncorrupted clients use to combine the server's estimate their own sample and output a final estimate of their component mean.

{\sc Robust Clustering} algorithm works in two phases. In the first phase it filters out `outliers' - those data points which have few other data points near them. Then, the algorithm clusters the filtered points based on their distance from each other. If two points are closer than $3\Delta$, then they belong to the same cluster.
These clusters give the server a rough idea of where the components are.
Since there are more total number of corrupted clients than uncorrupted clients in any one component, it may impossible to accurately locate all clusters. The job of {\sc Robust Clustering} is to force the adversary to place many corrupted data points near a component in order to successfully attack it. It does this by filtering aggressively - if no corrupted clients were present, only points very close to component means will get filtered in. This forces the adversary to concentrate the corrupted data points near a few components, leaving most of the components untouched.
As a result, a small fraction of the clients have to rely almost exclusively on their verified sample (thereby experiencing mean squared error close to 1), while the majority of the clients rely on the high quality estimate of the server (thereby experiencing mean squared error close to 0).
The adversarial strategy discussed in the lower bound also this property, which we have tried to replicate with our algorithm. 

In the second phase, the algorithm draws fresh samples and computes a mean estimate and a quality estimate for each cluster.
The mean estimate is the median of fresh samples in the vicinity of the points in the cluster. (For generalising to higher dimensions, a different robust mean estimator can be used.)
Since not every component will get a good estimate, it is important for the clients to know the quality of the server's estimate for their component's mean, otherwise the some client might rely heavily on corrupted data and thus incur a large error.
The quality estimate is based on counting the number points seen in an interval and comparing it to the number of genuine points expected if a component mean was located in that interval. 

To obtain a final estimate of some client's mean, {\sc Combine Estimates} simply takes a convex combination of the client's sample and the closest mean estimate. (If no mean estimate is close or if the closest cluster is of bad quality, then it uses the client's sample as the final estimate.)
The weights of the convex combination are based on the quality estimate of the closest cluster.
\begin{algorithm}[h]
\caption{Server's algorithm, $\mathcal{S}$}\label{alg:1}
	\begin{algorithmic}[1]
		\Function{Robust Clustering}{$\{ y_i \}_{i\in[n]}$, $k$, $c$, $ \color{red}{\delta}$}
        \State $m_h\leftarrow \frac{2}{\delta^2}\log(\frac{k}{\delta})$    \textcolor{black}{// $\delta$ is a backoff parameter which we asymptotically set to 0.}
        \State $m\leftarrow n/(k+c)$
        \State $m_1=m_2\leftarrow (m-m_h)/2$
	\State Randomly assign each point in      $\{ y_i \}_{i\in[n]}$ to $H$,        $T_1$ or $T_2$ with odds $m_h:m_1:m_2$
	\State $\rho \leftarrow \Phi\left( 
        \Delta \right) - \Phi\left( -\Delta\right)$
	\State $H'\leftarrow\{h\in H\mid$ 
        there are at least $(\rho-\delta) m_1$ points from $T_1$ less than $\Delta$ distance from $h\}$
	\State Let $G=(H',E)$ be a graph with 
        edges defined by $(h_1,h_2)\in E\iff |h_1-h_2|<3\Delta$
		\State $C\leftarrow$ the set of cliques in $G$
		\For{$c\in C$ }
		\State$S_c\leftarrow\{x\in T_2\mid \exists h\in c: |x-h|<\Delta\}$
		\State $l\leftarrow \max(S_c)-\min(S_c)-2\Delta$
		\State $\tilde{m}\leftarrow |S_c|$
		\State $\tilde{\mu}_c\leftarrow \mathrm{median}(S_c)$
		\State $\tilde{\varepsilon}_c \leftarrow \frac{\tilde{m}}{m}-\left(\Phi\left(l+\Delta\right)-\Phi\left(-\Delta\right)\right) + \delta$
		\EndFor
		\State\Return $L\leftarrow\{(\tilde{\mu}_c,\tilde{\varepsilon}_c)\mid c\in C\}$
		\EndFunction
	\end{algorithmic}
\end{algorithm}
\begin{algorithm}[h]
\caption{Client's algorithm, $\mathcal{C}$}\label{alg:2}
	\begin{algorithmic}[1]
		\Function{Combine Estimates}{$L$: output of Robust Clustering, $x_v$: verified sample}
		\State $(\tilde{\mu},\tilde{\varepsilon})\leftarrow \arg\min_{(\tilde{\mu},\tilde{\varepsilon})\in L}|\tilde{\mu}-x_v|$
		\If{$|\tilde{\mu}-x_v|<3\Delta/2$}
			\State\Return$\hat{\mu}\leftarrow\frac{1}{1+g(\tilde{\varepsilon})^2}\tilde{\mu} + \frac{g(\tilde{\varepsilon})^2}{1+g(\tilde{\varepsilon})^2}x_v$
		\Else
			\State\Return $x_v$
		\EndIf
		\EndFunction
	\end{algorithmic}
\end{algorithm}

For the asymptotic analysis:
\textcolor{black}{Adversarial points can either be included in a cluster and bias its mean, or elongate a cluster beyond $3\Delta$ and cause it to be discarded.}
We show that any adversarial point can affect at most one component in either of the two ways.
The fact that $D>9\Delta$ is used to show that any adversarial point close enough to a component to affect it is too far from any other components to affect them.
\begin{lemma}
\label{lm:1}
    For the Robust Clustering algorithm, conditioned on the realisation of the filtering phase, each adversarial point in the estimation phase can affect at most one component.
\end{lemma}
Next, we focus on any one component $i$. We bound the error in the server's mean estimate  $\tilde{\mu}_i$ of this component in terms of the mass of adversarial points affecting this component, $\varepsilon_i$. Due to Lemma \ref{lm:1}, we know that no adversarial point can affect multiple components, giving us $\sum_{i\in[k]}\varepsilon_i\le c$.
\begin{lemma}
\label{lm:median_error}
    Suppose $\varepsilon_i$ corrupted samples contribute to the median calculation for component $i$, then
    \begin{align*}
        &\mathbb{E}[(\mu_i-\tilde{\mu}_i)^2]\\
        &\le\Phi^{-1}\left(\frac{1}{2}\left(1+\varepsilon_i+\Phi\left(f^{-1}(\rho-\varepsilon_i)+\Delta\right)-\Phi\left(\Delta\right)\right)\right)^2\\
        &=:g(\varepsilon_i)^2
    \end{align*}
    where $\Phi$ is the standard normal c.d.f., $f(x):=\Phi(x+\Delta)-\Phi(x-\Delta)$
    and $\rho:=f(0)$.
\end{lemma}
Note that for a standard problem of Gaussian mean estimation from $\varepsilon$-additive corrupted samples, outputting the median is results in an MSE of $\Phi^{-1}\left(\frac{1}{2}\left(1+\varepsilon_i\right)\right)^2$.
Here we have an extra term $\Phi\left(f^{-1}(\rho-\varepsilon_i)+\Delta\right)-\Phi\left(\Delta\right)$ for the following reason: When {\sc Robust Clustering} calculates the median (as a mean estimate) it only considers data points in an interval containing the mean. This interval is, in general, not symmetric about the mean. This causes the additional term. Increasing $\Delta$ reduces the effect of this skewing of the interval of interest.

We now want to combine the server's estimate and the client's verified sample into a final local mean estimate. In order to give appropriate weights to the two estimates we would like to know $\mathbb{E}[(\mu_i-\tilde{\mu}_i)^2]$. If $\varepsilon_i$ were known this would be easy. Since $\varepsilon_i$ is unknown, we estimate it.
Since we have tight upper and lower bounds on the number of genuine points that can end up in an interval containing a component mean, we can get an estimate for the number of adversarial samples in the interval simply by counting the total number of points in the interval and subtracting the expected number of genuine points. This shows that $\tilde{\varepsilon}_i$ is a good estimate for $\varepsilon_i$.
\begin{lemma}
    \label{lm:epsilon_estimate}
    Suppose $\varepsilon_i$ corrupted samples contribute to the median calculation for component $i$. Let $\tilde{\varepsilon}_i$ be the server's estimate of $\varepsilon_i$ (see algorithm \ref{alg:1}, lines 10-12). Then, $ \varepsilon_i\le\tilde{\varepsilon}_i\le \varepsilon_i + f^{-1}(\rho-\varepsilon_i)\varphi(\Delta)=:\hat{\varepsilon}$.
    
    Here $\varphi$ is the standard normal p.d.f., and $f$ and $\rho$ are as in lemma \ref{lm:median_error}.
\end{lemma}
This gives us a bound on the mean squared error of the client's final estimate of its mean, $\hat{\mu}_i$.
\begin{lemma}
    \label{lm:one_comp_error}
    Suppose $\varepsilon_i$ corrupted samples contribute to the median calculation for component $i$, then, conditioned on the client's verified sample $x_v$ falling within $1.5\Delta$ distance to its mean, 
    \begin{align*}
        \mathbb{E}\left[(\hat{\mu}_i-\mu_i)^2\mid(x_v-\mu_i)^2<9/4\Delta^2\right]\le\frac{2g(\tilde{\varepsilon}_i)^2}{1+g(\tilde{\varepsilon}_i)^2}
    \end{align*}
    where $g$ is as defined in lemma \ref{lm:median_error} and $\tilde{\varepsilon}_i$ is the estimate of $\varepsilon_i$. (See Algorithm \ref{alg:1}, lines 10-12.)
\end{lemma}
Finally, we calculate the best way for the adversary to spread its budget of $cm$ points among the different components. That is we want to maximise
\[
    \sum_i\mathbb{E}\left[(\hat{\mu}_i-\mu_i)^2\mid(x_v-\mu_i)^2<9/4\Delta^2\right]
\]
subject to $\sum_i\varepsilon_i\le c$.
The adversary's strategy depends on whether $c/k$ is larger or smaller than a critical value $\varepsilon_{cr}$, which is the inflection point of  $\frac{g(\hat{\varepsilon}_i)^2}{1+g(\hat{\varepsilon}_i)^2}$ - the one-component MSE bound obtained by combining Lemmas \ref{lm:epsilon_estimate} and \ref{lm:one_comp_error}.
\begin{itemize}
    \item When $c/k$ is smaller than $\varepsilon_{cr}$, the best adversarial strategy is to attack all components equally (i.e. $\varepsilon_i=c/k$ for all $i$). In this case, all clients see the same MSE, which is at most $\frac{g(\tilde{\varepsilon}_i)^2}{1+g(\tilde{\varepsilon}_i)^2}$. To show its behavior near 0, we further upper bound this function by $C_2(\Delta)\left(\frac{c}{k}\right)^{2/9}$.
    \item Otherwise, the best strategy is to spread a small amount of mass equally among all components and use the rest to focus on just a few components. For roughly $c/k$ fraction of the components, the server's mean estimate is rendered useless resulting in an MSE of 1 for these components while the rest of the components are mostly unaffected incurring a very small MSE.
    Note that this is similar to the adversarial strategy in the lower bound; the difference being in how efficiently the adversary can use their corrupted samples. Our filtration step is  simple, and there could be more sophisticated filtration schemes with better performance; we considered several, but the challenge is in being able to give a complete analysis.
    Even this seemingly simple scheme requires a non-trivial analysis.
\end{itemize}
In the event that the client's verified sample falls far from its mean, we give an upper bound on its MSE which is proportional to $\Delta^2$. Since the probability of this event is exponentially (in $\Delta$) small, this gives a good bound on the MSE in this case.


\section{Discussion}
\label{sec:Disc}
We focused on a concrete problem of clients with samples from a cluster chosen from a scalar Gaussian mixture model. The natural heterogeneity arose from the different clusters chosen, and adversaries could arbitrarily corrupt interactions from clients. Even this seemingly simple case required innovation in robust filtering algorithms, their analysis and an appropriate lower bound. Using these, we were able to show that despite corruption, collaboration did better than using one's own local individual samples for personalized mean estimation. There are several open questions, even for this formulation, which are part of ongoing investigations. These include, analysis in the finite regime (including scaling), efficient high-dimensional estimation, and extensions with unknown and heterogeneous variances of the mixture components. There are also broader questions on how this can be extended to general distributions and higher dimensions with efficient algorithms. 

\printbibliography
\clearpage
\newpage
\appendices
\section{Lower Bound Proof}
From Section \ref{sec:LowerBnd}, we know that $K(m)$ is lower bounded by
\begin{align*}
	\max_{\{\varepsilon_i\geq 0\}_{i\in[k]}:\sum\varepsilon_i\leq c} \quad\frac{1}{k}\sum_{i\in[k]}h(\varepsilon_i).
\end{align*}
The function $h(\varepsilon)$ is convex for $\varepsilon\leq\sqrt{2\pi/3}$ and concave thereafter.
Informally, the best strategy for the adversary would be to bring as many cliques close to the inflection point as possible. Thus, the final lower bound on $K(m)$ would be proportional to $c/k$. Let us confirm this:
We want to find
\begin{align*}
    \max_{\varepsilon_i\in\mathbb{R}} \min_{\alpha,\beta_i\in\mathbb{R}^+} \sum_{i\in[k]}\frac{\varepsilon_i^2}{\varepsilon_i^2+2\pi}+\alpha\left(c-\sum_{i\in[k]}\varepsilon_i\right)+\sum_{i\in[k]}\beta_i\varepsilon_i
\end{align*}
Taking derivative w.r.t $\varepsilon_i$ gives
\begin{align*}
	\frac{d}{d\varepsilon_i} \bigg(\cdots \bigg)={}& \frac{4\pi\varepsilon_i} {\left(\varepsilon_i^2+2\pi\right)^2} - \alpha + \beta_i\\
	={}& \frac{1}{(\varepsilon_i^2+2\pi)^2}\left[(\beta_i-\alpha)(\varepsilon_i^2+2\pi)^2+4\pi\varepsilon_i\right]
\end{align*}
Setting this to 0 gives
\begin{align*}
	4\pi\varepsilon_i={}&(\alpha-\beta_i)(\varepsilon_i^2+2\pi)^2
\end{align*}
Thus, $\varepsilon_i=0\implies \beta_i=\alpha$. Furthermore, due to complementary slackness, $\varepsilon_i>0\implies\beta_i=0$.\\\\
In the optimal solution, some of the $\varepsilon_i$'s will be non-zero. w.l.o.g., let these be $\varepsilon_1,\ldots,\varepsilon_{k'}$ for some $k'\le k$. The other $\varepsilon_i$'s are 0. Then, we will have
\begin{align*}
	\frac{4\pi\varepsilon_i}{(\varepsilon_i^2+2\pi)^2}={}&h'(\varepsilon_i)=\alpha&\forall i\in[k']\\
	\varepsilon_i={}&0&\forall i\in[k]\backslash[k']
\end{align*}
Thus, all non-zero $\varepsilon_i$'s have the same $h'$.\\
Now, the function $h'(\varepsilon)$, for $\varepsilon\ge0$, is a 1-hump function with a maximum at $\varepsilon=\sqrt{2\pi/3}$. So, the equation $h'(\varepsilon)=\alpha$ has 2 solutions. Call them $h'^{-1}_-(\alpha)$ and $g'^{-1}_+(\alpha)$.\\
Suppose we have a stationary point with some $\alpha$, $\varepsilon_1=h'^{-1}_-(\alpha),\varepsilon_2=h'^{-1}_+(\alpha)$ (And there may be some other non-zero $\varepsilon_i$'s among $3\le i\le k$ which may be equal to $h'^{-1}_-(\alpha)$ or $h'^{-1}_-(\alpha)$).\\
We claim that this is not a local maximum. Specifically, we claim that the Hessian in the direction $[-1,1,0,\ldots,0]$ is positive.\\
That is,
\begin{align*}
	h''(h'^{-1}_+(\alpha))+h''(h'^{-1}_-(\alpha))>{}&0\\
	\iff4\pi\left[\frac{2\pi-3h'^{-1}_+(\alpha)}{(2\pi+h'^{-1}_+(\alpha)^2)^3} + \frac{2\pi-3h'^{-1}_-(\alpha)}{(2\pi+h'^{-1}_-(\alpha)^2)^3}\right]>{}&0
\end{align*}
We can numerically confirm this. Then the only local maxima are with $\varepsilon_i=h'^{-1}_+(\alpha)\ \forall i\in[k']$. Since all non-zero $\varepsilon_i$'s are equal, they must all be equal to $\frac{c}{k'}$. The total error is lower bounded by
\begin{align*}
	K(m)\ge{}& \frac{1}{k}\sum_{i\in[k]}h(\varepsilon_i)\\
	={}& \frac{1}{k}\sum_{i\in[k']}h\left(\frac{c}{k'}\right)\\
	={}& \frac{k'}{k}h\left(\frac{c}{k'}\right)\\
	={}& \frac{k'}{k}\frac{(c/k')^2}{(c/k')^2+2\pi}\\
	={}& \frac{1}{k}\frac{c^2k'}{c^2+2\pi k'^2}
\end{align*}
The algorithm will pick the best $k'$. So we get a lower bound of
\begin{align*}
	K(m)\ge{}& \frac{1}{k}\max_{k'\in\mathbb{Z^+}}\frac{c^2k'}{c^2+2\pi k'^2}\\
	\ge{}& \frac{1}{k}\max_{k'\in\mathbb{R}^+}\frac{c^2k'}{c^2+2\pi k'^2}\\
	={}& \frac{}{\sqrt{8\pi}}\frac{c}{k}
\end{align*}
Note that the maximum occurs at $k'=c/\sqrt{2\pi}$. Since we work with $c<\sqrt{2\pi}k$ the constraint $k'\le k$ is satisfied,
yielding a lower bound of $\frac{1}{\sqrt{8\pi}}\frac{c}{k}$.
\section{Algorithm analysis}
\begin{proof}[Proof of lemma \ref{lm:1}]
    To give an MSE bound for this algorithm, we need to consider the worst case scenario in terms of the adversary's choices (i.e. maximising the MSE over adversary's choices). The adversary can choose a configuration of $\{\mu_i\}_{i\in[k]}$ and location of corrupted points $\{z_j\}_{j\in[cm]}$. To make the maximisation tractable, we want to associate with each configuration $\{\mu_i\},\{z_j\}$, numbers $c_1$ and $\{\varepsilon_i\}_{i\in[k]}$ such that the algorithm's error can be bounded in terms of $c_1,\{\varepsilon_i\}$. Thus, we can reduce the maximisation over all of adversary's choices to just (constrained) maximisation over $c_1,\{\varepsilon_i\}$. Informally, $c_1$ is the mass of points used by the adversary to destroy clusters. $\varepsilon_i$ is the mass of points used by the adversary to shift the mean estimate of the $i$-th component\\
In particular, for each configuration of $\{\mu_i\},\{z_j\}$, we give numbers $c_1,\{\varepsilon_i\}$, such that, with high probability
\begin{enumerate}
	\item $c_1+\sum_i\varepsilon_i\le c(1+o(1))$.
	\item At most $\frac{c_1}{\tilde{\rho}}$ components are non-cliques.
	\item There are at most $\varepsilon_i$ adversarial points contributing to the $i$-th component, if it forms a clique.
\end{enumerate}
The algorithm is given a collection of $(k+c)m$ points, $cm$ of which are adversarial, in some order.\\
The algorithm will randomly assign each point to phase 1 with probability $\lambda$ independently of the other points. Let the number of genuine points assigned to phases 1 and 2 be $km_1$ and $km_2$ respectively and the number of adversarial points assigned to phases 1 and 2 be $cm_1'$ and $cm_2'$ respectively.\\
Thus, $m_1+m_2=m_1'+m_2'=m$, $\mathbb{E}[m_1]=\mathbb{E}[m_1']=\lambda m$, and $\mathbb{E}[m_2]=\mathbb{E}[m_2']=(1-\lambda)m$.\\\\
We want $c_1,\{\varepsilon_i\}_{i\in[k]}$ to depend only on what happens in phase 1 and not on the realisation of genuine points from phase 2.
Conditioned on the phase assignment of the adversarial points and on the realisation of the phase 1 genuine points, the outcome of phase 1 - that is, the graph $G$ the algorithm comes up with - is determined. Then, we can simply count how many phase 2 adversarial points are part of each clique/non-clique. We count them under $\varepsilon_i$ if they are part of a clique corresponding to component $i$ and under $c_1$ if they are part of some non-clique. This satisfies conditions 1 and 3.\\\\
$c_1$ is defined as (number of adversarial points in non-cliques)/$(1-\lambda)m$.\\
Similarly, $\varepsilon_i$ is defined as (number of adversarial points in the clique for component $i$)/$(1-\lambda)m$.\\
Clearly,
\begin{align*}
	c_1+\sum_i\varepsilon_i\le c\frac{m_2'}{(1-\lambda)m}=c\frac{m_2'}{\mathbb{E}[m_2']}
\end{align*}
As $m\rightarrow\infty$, $\frac{m_2'}{\mathbb{E}[m_2']}\le 1+\delta$ with high probability for any small $\delta>0$. Thus, condition 1 is satisfied.
\end{proof}


\begin{proof}[Proof of lemma \ref{lm:median_error}]
Define $f(x):= \Phi(x+\Delta)-\Phi(x-\Delta)$.
Note that $f(x)$ measures the expected mass of genuine samples within $\Delta$ distance of a point $x$ units from the mean of a component.
Suppose the adversary has placed a sample at $x$. For it to be filtered in the total number of points within $\Delta$ units of it must be $(\rho-\delta)m$. Since the number of genuine samples is a little over $f(x)m$, the number of adversarial points needed is $(\rho-f(x)-\delta)m$. (In the large sample limit $m\rightarrow\infty$, $\delta\rightarrow0$)
To put it another way, if the adversary attacks a component with $\varepsilon m$ samples, the furthest point from the mean that can get filtered in is $f^{-1}(\rho-\varepsilon)$ units away.

Suppose the clique has $\varepsilon_i m$ adversarial points and $\alpha_i m$ genuine points. If $\varepsilon_i<\alpha_i$, the median is bounded between the $\left( \frac{1}{2} - \frac{\varepsilon_i}{2\alpha_i} \right)$-th and $\left( \frac{1}{2} + \frac{\varepsilon_i}{2\alpha_i} \right)$-th percentile of the genuine points.

Suppose the adversary wants to shift the mean to the right, then the worst truncation interval is $(-\Delta,f^{-1}(\rho-\varepsilon)+\Delta)$.The $\left( \frac{1}{2} + \frac{\varepsilon_i}{2\alpha_i} \right)$-th percentile of such a distribution is
\begin{align*}
	&\sigma\Phi^{-1}\left(\Phi(-\Delta/\sigma) + \left(\frac{1}{2}+\frac{\varepsilon_i}{2\alpha_i}\right)\alpha_i\right)\\
	\le{}&\sigma\Phi^{-1}\left(\Phi(-\Delta/\sigma) +\frac{\varepsilon_i}{2} +\frac{1}{2}\left( \Phi\left(\frac{f^{-1}(\rho-\varepsilon_i)+\Delta}{\sigma}\right)\right.\right.\\
    &\qquad\left.\left.-\Phi\left(-\frac{\Delta}{\sigma}\right)\right)\right)\\
	={}& \sigma\Phi^{-1}\left(\frac{1}{2}+\frac{\varepsilon_i}{2}+\frac{1}{2}\left(\Phi\left(\frac{f^{-1}(\rho-\varepsilon_i)+\Delta}{\sigma}\right)-\Phi\left(\frac{\Delta}{\sigma}\right)\right)\right)
\end{align*}
\end{proof}

\begin{proof}[Proof of lemma \ref{lm:epsilon_estimate}]
    Let the length of the interval over which we are taking the median be $l+2\Delta$ and the number of points contributing to the median be $\tilde{m}$. We know that the number of genuine points in the interval is between $\Phi\left(l/2+\Delta\right) - \Phi\left(-(l/2+\Delta)\right)$ and $\Phi\left(l+\Delta\right)-\Phi(-\Delta)$. Thus,
\begin{align*}
	\tilde{m}-m\left(\Phi\left(l+\Delta\right)-\Phi\left(-\Delta\right)\right)\ge \varepsilon m\\
    \ge \tilde{m}-m\left( \Phi\left(l/2+\Delta\right) - \Phi\left(-(l/2+\Delta)\right)\right)
\end{align*}
We also know
\begin{align*}
	\varepsilon\ge{}& \rho-f(l/2)\\
	\implies l\le{}& 2f^{-1}(\rho-\varepsilon)
\end{align*}
A conservative estimate of $\varepsilon$ is
\begin{align*}
	\hat{\varepsilon}={}& \frac{\tilde{m}}{m}-\left(\Phi\left(l+\Delta\right)-\Phi\left(-\Delta\right)\right)
\end{align*}
This gives,
\begin{align*}
	\varepsilon\le \hat{\varepsilon}={}& \frac{\tilde{m}}{m}-\left(\Phi\left(l+\Delta\right)-\Phi\left(-\Delta\right)\right)\\
	\le{}& \varepsilon - \left(\Phi\left(l+\Delta\right) - \Phi\left(-\Delta\right)\right) + \left(\Phi\left(l/2+\Delta\right) - \Phi\left(-(l/2+\Delta)\right)\right)\\
	={}& \varepsilon - \left(\Phi\left(l+\Delta\right)-\Phi\left(l/2+\Delta\right)\right) - \left(\Phi\left(-(l/2+\Delta\right)-\Phi\left(-\Delta\right)\right)\\
	={}& \varepsilon - \left(\Phi\left(l+\Delta\right)-\Phi\left(l/2+\Delta\right)\right) + \left(\Phi\left(l/2+\Delta\right)-\Phi\left(\Delta\right)\right)\\
	\le{}& \varepsilon + \frac{1}{\sqrt{2\pi}}l\left(e^{-\frac{\Delta^2}{2}} - e^{-\frac{(l+\Delta)^2}{2}}\right)\\
	\le{}& \varepsilon +  \frac{l}{\sqrt{8\pi}\sigma} e^{-\frac{\Delta^2}{2\sigma^2}}\\
    \le{}& \varepsilon + f^{-1}(\rho-\varepsilon)\varphi(\Delta)
\end{align*}
\end{proof}
\begin{proof}[Proof of lemma \ref{lm:one_comp_error}]
    \begin{align*}
	&\mathbb{E}\left[(\hat{\mu}_i-\mu_i)^2\mid(x_v-\mu_i)^2<9/4\Delta^2\right]\\
    ={}& \frac{1}{(1+g(\tilde{\varepsilon})^2)^2}\mathbb{E}\left[(\bar{\mu}_i-\mu_i)^2\right] \\
    &\quad+\frac{g(\tilde{\varepsilon})^4}{(1+g(\tilde{\varepsilon})^2)^2}\mathbb{E}\left[(x_v-\mu_i)^2\mid(x_v-\mu_i)^2<9/4\Delta^2\right]\\
    \le{}&\frac{1}{(1+g(\tilde{\varepsilon})^2)^2}\mathbb{E}\left[(\bar{\mu}_i-\mu_i)^2\right] \\
    &\quad+\frac{g(\tilde{\varepsilon})^4}{(1+g(\tilde{\varepsilon})^2)^2}\mathbb{E}\left[(x_v-\mu_i)^2\right]\\
    ={}& \frac{g(\varepsilon)^2 + g(\hat{\varepsilon})^4}{(1+g(\hat{\varepsilon})^2)^2}\\
    \le{}& \frac{g(\hat{\varepsilon})^2 + g(\hat{\varepsilon})^4}{(1+g(\hat{\varepsilon})^2)^2}\\
	={}& \frac{g(\hat{\varepsilon})^2}{1+g(\hat{\varepsilon})^2}
\end{align*}
\end{proof}
Finally, we calculate the best way for the adversary to spread its budget of $cm$ points among the different components.
The adversary's strategy depends on whether $c/k$ is larger or smaller than a critical value $\varepsilon_{cr}$, which is the inflection point of  $\frac{g(\hat{\varepsilon}_i)^2}{1+g(\hat{\varepsilon}_i)^2}$ - the one-component MSE bound obtained by combining Lemmas \ref{lm:epsilon_estimate} and \ref{lm:one_comp_error}. 
\begin{itemize}
    \item When $c/k$ is smaller than $\varepsilon_{cr}$, the best adversarial strategy is to attack all components equally (i.e. $\varepsilon_i=c/k$ for all $i$). In this case, all clients see the same MSE, which is at most $\frac{g(\tilde{\varepsilon}_i)^2}{1+g(\tilde{\varepsilon}_i)^2}$. To show its behavior near 0, we further upper bound this function by $C_2(\Delta)\left(\frac{c}{k}\right)^{2/9}$.
    \item Otherwise, the best strategy is to spread a small amount of mass equally among all components and use the rest to focus on just a few components. For roughly $c/k$ fraction of the components, the server's mean estimate is rendered useless resulting in an MSE of 1 for these components while the rest of the components are mostly unaffected incurring a very small MSE.
    Note that this is similar to the adversarial strategy in the lower bound; the difference being in how efficiently the adversary can use their corrupted samples. Our filtration step is extremely simple, and as such more susceptible to corruptions. A more sophisticated filtration scheme may perform better.
\end{itemize}
Let us now prove this.
\begin{lemma}
	The function $\frac{g(\hat{\varepsilon})^2}{1+g(\hat{\varepsilon})^2}$ is convex for
    $\varepsilon>\varepsilon_{cr}$ where $\varepsilon_{cr}$ is the positive zero of $(1-g(\hat{\varepsilon})^2)^2 + (1+g(\hat{\varepsilon})^2)  \frac{g(\hat{\varepsilon})}{g'(\hat{\varepsilon})} \Delta l'(\hat{\varepsilon}) \left(\frac{l'(\varepsilon)}{\hat{\varepsilon}'^2}\varphi\left(\Delta\right) \left(\frac{l(\varepsilon)}{\Delta} - \coth(l(\varepsilon)\Delta) \right)\right.$ $\left. -  \frac{e^{-l(\hat{\varepsilon})\Delta}}{\sinh(l(\tilde{\varepsilon})\Delta)}\right)$.
	
	Here,
	\begin{align*}
		f(x):={}&\Phi\left(x+\Delta\right)-\Phi\left(x-\Delta\right)\\
		\rho:={}&f(0)\\
		l(\varepsilon):={}& f^{-1}(\rho-\varepsilon)\\
		g(\varepsilon):={}&\Phi^{-1}\left(\frac{1}{2}\left(1+\varepsilon+\left[\Phi\left(l(\varepsilon)+\Delta\right)-\Phi\left(\Delta\right)\right]\right)\right)\\
		\hat{\varepsilon}:={}& \varepsilon + \varphi(\Delta)l(\varepsilon)
	\end{align*}
\end{lemma}
\begin{proof}
First, note that
    \begin{align*}
		l'(\varepsilon)={}& \frac{-1}{f'(l(\varepsilon))}\\
		={}& \frac{1}{\varphi\left(l(\varepsilon)-\Delta \right) - \varphi\left( l(\varepsilon)+\Delta \right)}\\
		={}& \frac{\sqrt{2\pi}} {e^{-\frac{(l(\varepsilon)-\Delta)^2}{2}} - e^{-\frac{(l(\varepsilon)+\Delta)^2}{2}}}\\
		={}& \frac{\sqrt{2\pi} e^{\frac{l(\varepsilon)^2+\Delta^2}{2}}}{2\sinh(l(\varepsilon)\Delta)}
	\end{align*}
 and
    \begin{align*}
		g'(\varepsilon)={}& \frac{1}{2\varphi(g(\varepsilon))}\left(1+\varphi\left(l(\varepsilon)+\Delta\right)l'(\varepsilon)\right)\\
		={}&\frac{1}{2\varphi(g(\varepsilon))} \left(1+\varphi\left(l(\varepsilon)+\Delta\right)\frac{1}{\varphi\left(l(\varepsilon)-\Delta \right) - \varphi\left(l(\varepsilon)+\Delta\right)} \right)\\
		={}&\frac{1}{2\varphi(g(\varepsilon))}\frac{\varphi\left(l(\varepsilon)-\Delta\right)}{\varphi\left(l(\varepsilon)-\Delta \right) - \varphi\left(l(\varepsilon)+\Delta\right)}\\
		={}&\frac{1}{2\varphi(g(\varepsilon))}\frac{1}{1-e^{-2l(\varepsilon)\Delta}}\\
		={}& \frac{e^{l(\varepsilon)\Delta}} {4\varphi(g(\varepsilon))\sinh(l(\varepsilon)\Delta)}
	\end{align*}
    \begin{align*}
		g''(\varepsilon)={}& g'(\varepsilon)\left[\Delta l'(\varepsilon) - \frac{\varphi'(g(\varepsilon))g'(\varepsilon)}{\varphi(g(\varepsilon))} - \coth(\Delta l(\varepsilon))\frac{\Delta l'(\varepsilon)}{\sigma^2}\right]\\
		={}& g'(\varepsilon)\left[\Delta l'(\varepsilon) + g(\varepsilon)g'(\varepsilon) - \coth(\Delta l(\varepsilon))\Delta l'(\varepsilon)\right]
	\end{align*}
 Now,
    \begin{align*}
		\frac{d}{d\varepsilon}\left[\frac{g(\tilde{\varepsilon})^2}{1+g(\tilde{\varepsilon})^2}\right]	={}& \frac{d}{d(g^2)}\left[1-\frac{1}{1+g(\tilde{\varepsilon})^2}\right]\cdot(g(\tilde{\varepsilon})^2)'\\
		={}& \frac{(g(\tilde{\varepsilon})^2)'}{(1+g(\tilde{\varepsilon})^2)^2}\\
		\frac{d}{d\varepsilon}\frac{(g(\tilde{\varepsilon})^2)'}{(1+g(\tilde{\varepsilon})^2)^2}	={}&  \frac{(1+g(\tilde{\varepsilon})^2)(g(\tilde{\varepsilon})^2)''-2(g(\tilde{\varepsilon})^2)'^2}{(1+g(\tilde{\varepsilon})^2)^3}
	\end{align*}
 This has the same sign as
	\begin{align*}
		&(1/2)(1+g(\tilde{\varepsilon})^2)(g(\tilde{\varepsilon})^2)''-(g(\tilde{\varepsilon})^2)'^2\\
		={}& (1+g(\tilde{\varepsilon})^2) ((g'(\tilde{\varepsilon})^2+g(\tilde{\varepsilon})g''(\tilde{\varepsilon})) \tilde{\varepsilon}'^2 + g(\tilde{\varepsilon})g'(\tilde{\varepsilon})\tilde{\varepsilon}'')\\
        &- (2g(\tilde{\varepsilon}) g'(\tilde{\varepsilon})\tilde{\varepsilon}')^2\\
		={}&g'(\tilde{\varepsilon})^2\tilde{\varepsilon}'^2\left[(1+g(\tilde{\varepsilon})^2)\left(1+\frac{g(\tilde{\varepsilon})g''(\tilde{\varepsilon})}{g'(\tilde{\varepsilon})^2}+\frac{g(\tilde{\varepsilon})}{g'(\tilde{\varepsilon})}\frac{\tilde{\varepsilon}''}{\tilde{\varepsilon}'^2}\right) - 4g(\tilde{\varepsilon})^2\right]
	\end{align*}
 	This has the same sign as
	
 	\begin{align*}
		&(1+g(\hat{\varepsilon})^2)\left(1+\frac{g(\hat{\varepsilon})g''(\hat{\varepsilon})}{g'(\hat{\varepsilon})^2} +\frac{g(\hat{\varepsilon})}{g'(\hat{\varepsilon})} \frac{\hat{\varepsilon}''}{\hat{\varepsilon}'^2} \right) - 4g(\hat{\varepsilon})^2\\
		={}&(1+g(\hat{\varepsilon})^2)\left(1+\frac{g(\hat{\varepsilon})}{g'(\hat{\varepsilon})}\bigg(\Delta l'(\hat{\varepsilon}) + g(\hat{\varepsilon}) g'(\hat{\varepsilon}) \bigg. \right.\\
        &\left. \bigg. - \coth(\Delta l(\hat{\varepsilon}))\Delta l'(\hat{\varepsilon})\bigg) + \frac{g(\hat{\varepsilon})}{g'(\hat{\varepsilon})}\frac{\hat{\varepsilon}''}{\hat{\varepsilon}'^2}\right) - 4g(\hat{\varepsilon})^2\\
		={}& (1+g(\hat{\varepsilon})^2) \left(1 + g(\hat{\varepsilon})^2\right.\\
        & \left.+  \frac{g(\hat{\varepsilon})}{g'(\hat{\varepsilon})}\left( \Delta l'(\hat{\varepsilon}) - \coth(\Delta l(\hat{\varepsilon})) \Delta l'(\hat{\varepsilon}) + \frac{\hat{\varepsilon}''}{\hat{\varepsilon}'^2}\right)\right) - 4g(\hat{\varepsilon})^2\\
		={}&(1-g(\hat{\varepsilon})^2)^2 \\
        &\!\!\!\!+(1+g(\hat{\varepsilon})^2) \left( \frac{g(\hat{\varepsilon})}{g'(\hat{\varepsilon})} \left( \Delta l'(\hat{\varepsilon}) - \coth(\Delta l(\hat{\varepsilon})) \Delta l'(\hat{\varepsilon}) + \frac{\hat{\varepsilon}''}{\hat{\varepsilon}'^2}\right)\right)\\
		={}& (1-g(\hat{\varepsilon})^2)^2\\
        &+ (1+g(\hat{\varepsilon})^2) \left( \frac{g(\hat{\varepsilon})}{g'(\hat{\varepsilon})} \left( \Delta l'(\hat{\varepsilon})\left(1 - \coth(\Delta l(\hat{\varepsilon}))\right) + \frac{\hat{\varepsilon}''}{\hat{\varepsilon}'^2}\right)\right)\\
		={}& (1-g(\hat{\varepsilon})^2)^2\\
        &+ (1+g(\hat{\varepsilon})^2) \left( \frac{g(\hat{\varepsilon})}{g'(\hat{\varepsilon})} \left( \frac{\hat{\varepsilon}''}{\hat{\varepsilon}'^2} - \Delta l'(\hat{\varepsilon}) \frac{2e^{-l(\hat{\varepsilon})\Delta}}{e^{l(\hat{\varepsilon})\Delta} - e^{-l(\hat{\varepsilon})\Delta}} \right)\right)\\
		={}& (1-g(\hat{\varepsilon})^2)^2\\
        &+ (1+g(\hat{\varepsilon})^2) \left( \frac{g(\hat{\varepsilon})}{g'(\hat{\varepsilon})} \left( \frac{l'(\varepsilon)^2\Delta}{\hat{\varepsilon}'^2}\varphi\left(\Delta\right) \left(\frac{l(\varepsilon)}{\Delta} - \coth(l(\varepsilon)\Delta) \right)\right.\right.\\
        &\left.\left.- \Delta l'(\hat{\varepsilon}) \frac{e^{-l(\hat{\varepsilon})\Delta}}{\sinh(l(\hat{\varepsilon})\Delta)} \right)\right)\\
		={}& (1-g(\hat{\varepsilon})^2)^2+ (1+g(\hat{\varepsilon})^2)  \frac{g(\hat{\varepsilon})}{g'(\hat{\varepsilon})} \Delta l'(\hat{\varepsilon})\\
        &\cdot\left(\frac{l'(\varepsilon)}{\hat{\varepsilon}'^2}\varphi\left(\Delta\right) \left(\frac{l(\varepsilon)}{\Delta} - \coth(l(\varepsilon)\Delta) \right) -  \frac{e^{-l(\hat{\varepsilon})\Delta}}{\sinh(l(\hat{\varepsilon})\Delta)}\right) 
	\end{align*}
    The first term $(1-g(\tilde{\varepsilon})^2)^2$ is clearly non-negative; the term $\frac{l(\varepsilon)}{\Delta} - \coth(l(\varepsilon)\Delta)$ is increasing without bound; and the term  $ -  \frac{e^{-l(\tilde{\varepsilon})\Delta}}{\sinh(l(\tilde{\varepsilon}\Delta))}$ is increasing, and asymptotically tending to $0$. 
	
	Therefore, there exists $0<\varepsilon_{cr}<\rho$ such that for $\varepsilon>\varepsilon_{cr}$, the above function and hence the second derivative of $\frac{g(\tilde{\varepsilon})^2}{1+g(\tilde{\varepsilon})^2}$ is positive.
    Numerical calculations show that this $\varepsilon_{cr}$ is very small - for $\Delta= 2,3,4$, $\varepsilon_{cr} \approx 8.1\cdot10^{-6}, 1.2\cdot10^{-9},1.4\cdot10^{-14}$ respectively.
\end{proof}
When the loss function is concave, the adversary's best strategy is to split their points evenly across all components. Conversely, when the loss function is convex, the best strategy is to invest in as few components as possible. (Note that, investing more than $\rho-f(3\Delta)$ points in any one component does not make sense for the adversary).
Thus, when $c/k<\varepsilon_{cr}$, adversary will invest evenly across all clusters and the overall MSE will be bounded by 
\begin{align*}
    \lim_{m\rightarrow\infty} K(m) \le{}& \frac{g\left(\frac{c}{k} + l\left(\frac{c}{k} \right) \varphi(\Delta)\right)^2}{1+g\left(\frac{c}{k}+ l\left(\frac{c}{k} \right)\varphi(\Delta)\right)^2}
\end{align*}
We can bound this further by first studying the behaviour of $l(\varepsilon)$ for small $\varepsilon$. This is done in the following
\begin{lemma}
    $l(\varepsilon)\le \sqrt[3]{\varepsilon/C}$ for all $\varepsilon\le\varepsilon_{cr}$, where $C= \frac{\varepsilon_{cr}} {f^{-1}(\rho-\varepsilon_{cr})^3}$.
\end{lemma}
\begin{proof}
    Since we are considering the small $\varepsilon$ regime, consider first the Talylor expansion of $f(x)$ near $x=0$.
    \begin{align*}
        f(x) ={}& \Phi(x+\Delta)-\Phi(x-\Delta)\\
        ={}& \rho +2\varphi'(x+\Delta)x^2+ \mathcal{O}(x^4)
    \end{align*}
    So, if we define $h(x):=\rho-Cx^3$, then for small enough $x$, $f(x)\le h(x)$.
    Let $x_{cr}=l(\varepsilon_{cr})$.
    Choose $C$ such that $\forall x\in [0,x_{cr}]$, $f(x)\le h(x)\implies \rho-f(x)\ge \rho-h(x)$.
    Then, $f^{-1}(\rho-\varepsilon)\le h^{-1}(\rho-\varepsilon)$ for all $\varepsilon\in[0,\varepsilon_{cr}]$.
\end{proof}
This results in the following bound, using $l(c/k)\le \sqrt[3]{c/kC}$ and $c/k\le \varepsilon_{cr}^{2/3}\sqrt[3]{c/k}$:
\begin{align*}
    &\lim_{m\rightarrow\infty} K(m)\\
    \le{}& \frac{g\left(\frac{c}{k} + l\left(\frac{c}{k} \right) \varphi(\Delta)\right)^2}{1+g\left(\frac{c}{k}+ l\left(\frac{c}{k} \right)\varphi(\Delta)\right)^2}\\
    \le{}& \Phi^{-1}\left(\frac{1}{2}\left(1 + \frac{c}{k} + l\left(\frac{c}{k}\right)\varphi(\Delta) + l\left(\frac{c}{k} + l\left(\frac{c}{k}\right)\varphi(\Delta)\right)\varphi(\Delta)\right)\right)^2\\
    \le{}&\Phi^{-1}\left(\frac{1}{2}\left(1 + \left(\varepsilon_{cr}^{2/3}+\frac{\varphi(\Delta)}{\sqrt[3]{C}}\right)\sqrt[3]{\frac{c}{k}}\right.\right.\\
    &\left.\left.+ l\left(\left(\varepsilon_{cr}^{2/3}+\frac{\varphi(\Delta)}{\sqrt[3]{C}}\right)\sqrt[3]{\frac{c}{k}}\right)\varphi(\Delta)\right)\right)^2\\
    \le{}&\Phi^{-1}\left(\frac{1}{2}\left(1 + \left(\varepsilon_{cr}^{2/3}+\frac{\varphi(\Delta)}{\sqrt[3]{C}}\right)\varepsilon_{cr}^{2/9}\sqrt[9]{\frac{c}{k}}\right.\right.\\
    &\left.\left.+ \frac{\varphi(\Delta)}{\sqrt[3]{C}}\left(\varepsilon_{cr}^{2/3}+\frac{\varphi(\Delta)}{\sqrt[3]{C}}\right)^{1/3}\sqrt[9]{\frac{c}{k}}\right)\right)^2\\
    =:{}&\Phi^{-1}\left(\frac{1}{2}+\bar{C}(\Delta)\sqrt[9]{\frac{c}{k}}\right)^2\\
    \le{}& \left(\frac{\bar{C}(\Delta)\sqrt[9]{\frac{c}{k}}}{\varphi\left(\bar{C}(\Delta)\sqrt[9]{\varepsilon_{cr}}\right)}\right)^2\\
    ={}&\frac{\bar{C}(\Delta)^2}{\varphi\left(\bar{C}(\Delta)\sqrt[9]{\varepsilon_{cr}}\right)^2}\left(\frac{c}{k}\right)^{2/9} =: C_2(\Delta)\left(\frac{c}{k}\right)^{2/9}
\end{align*}
where $\bar{C}(\Delta):=\left(\varepsilon_{cr}^{2/3}+\frac{\varphi(\Delta)}{\sqrt[3]{C}}\right)\frac{\varepsilon_{cr}^{2/9}}{2}+\frac{\varphi(\Delta)}{2\sqrt[3]{C}}\left(\varepsilon_{cr}^{2/3}+\frac{\varphi(\Delta)}{\sqrt[3]{C}}\right)^{1/3}$ and $C_2(\Delta):= \frac{\bar{C}(\Delta)^2}{\varphi\left(\bar{C}(\Delta)\sqrt[9]{\varepsilon_{cr}}\right)^2}$.

When $c/k>\varepsilon_{cr}$, we will have $\varepsilon_i\le\varepsilon_{cr}$ for some $i$ and $\varepsilon_i>\varepsilon_{cr}$ for some other $i$. The contribution from the former is at most $\frac{g(\hat{\varepsilon}_{cr})^2}{1+g(\hat{\varepsilon_{cr}})^2}$.
For the latter, since we are now in the convex regime, the best strategy is to focus on a few components.
Note that investing $f^{-1}(\rho-3\Delta)m$ points in any component is enough to render that component's mean estimate useless.
So, $\frac{c}{f^{-1}(\rho-3\Delta)}$ many components will incur an error of 1 while the rest them will have an error of at most $\frac{g(\hat{\varepsilon}_{cr})^2}{1+g(\hat{\varepsilon_{cr}})^2}$. This brings the final error bound to
\begin{align*}
    \lim_{m\rightarrow\infty}K(m)\le{}& \frac{1}{f^{-1}(\rho-3\Delta)}\frac{c}{k} + \frac{g(\hat{\varepsilon}_{cr})^2}{1+g(\hat{\varepsilon_{cr}})^2}
\end{align*}
for the case $c/k>\varepsilon_{cr}$.

In the event that the client's verified sample $x_v$ falls far from its mean, we give an upper bound on its MSE which is proportional to $\Delta^2$. Since the probability of this event is exponentially (in $\Delta$) small, this gives a good bound on the MSE in this case. The issue is that if the verified sample falls close to a different component, the client might trust that component's mean estimate as if it's their own. However, even when this happens, the foreign component's mean estimate cannot be more than $1.5\Delta$ units away from $x_v$. Thus,
\begin{align*}
	&\mathbb{E}[|x_v-\mu_i|^2\mid (x_v-\mu_i)^2\ge9/4\Delta^2]\\
    \le{}& \frac{\int_{1.5\Delta}^{\infty} (x+1.5\Delta)^2 e^{-\frac{x^2}{2}}dx} {\int_{1.5\Delta}^{\infty}e^{-\frac{x^2}{2}}dx}\\
	={}& \frac{\int_{1.5\Delta}^{\infty} x^2 e^{-\frac{x^2}{2}}dx + \int_{1.5\Delta}^{\infty} 3x\Delta e^{-\frac{x^2}{2}}dx + \int_{1.5\Delta}^{\infty} (1.5\Delta)^2 e^{-\frac{x^2}{2}}dx} {\int_{1.5\Delta}^{\infty} e^{-\frac{x^2}{2}}dx}\\
	={}& \frac{-\left[xe^{-\frac{x^2}{2}}\right]_{1.5\Delta}^{\infty} + \int_{1.5\Delta}^{\infty} 3x\Delta e^{-\frac{x^2}{2}}dx}{\int_{1.5\Delta}^{\infty} e^{-\frac{x^2}{2}}dx} + 1+2.25\Delta^2\\
    ={}& \frac{-\left[(x+3\Delta)e^{-\frac{x^2}{2}}\right]_{1.5\Delta}^{\infty}}{\int_{1.5\Delta}^{\infty} e^{-\frac{x^2}{2}}dx} + 1+2.25\Delta^2\\
    \le{}& 3.25\Delta^2
\end{align*}

\end{document}